\documentclass[sigconf]{acmart}
\usepackage{algpseudocode}
\usepackage{algorithm}
\usepackage{booktabs}
\usepackage{multirow}
\usepackage{xcolor} 
\usepackage{colortbl}
\usepackage{appendix}
\usepackage{makecell}
\usepackage{tabularx}
\usepackage{dblfloatfix}
\newcommand{\M}{DiffTSP}
\AtBeginDocument{%
  }

\setcopyright{acmlicensed}
\copyrightyear{2018}
\acmYear{2018}
\acmDOI{XXXXXXX.XXXXXXX}
\acmConference[Conference acronym 'XX]{Make sure to enter the correct
  conference title from your rights confirmation email}{June 03--05,
  2018}{Woodstock, NY}
\acmISBN{978-1-4503-XXXX-X/2018/06}

\begin{document}

\title{DiffTSP: A Conditional Discrete Diffusion Model for \\ Knowledge Graph  Triple Set Prediction}



\author{Jihong Guan}
\orcid{0000-0003-2313-7635}
\affiliation{%
  \institution{Tongji University}
  \city{Shanghai}
  \country{China}}
  
\author{Jiaqi Wang}
\orcid{0009-0008-7293-3156}
\affiliation{%
  \institution{Tongji University}
  \city{Shanghai}
  \country{China}}
  
\author{Han Cai}
\affiliation{%
  \institution{Tongji University}
  \city{Shanghai}
  \country{China}}
  
\author{Wengen Li}
\orcid{0000-0002-8768-6740}
\affiliation{%
  \institution{Tongji University}
  \city{Shanghai}
  \country{China}}

\author{Hanchen Yang}
\orcid{0000-0002-9011-0355}
\affiliation{%
  \institution{Tongji University}
  \city{Shanghai}
  \country{China}}

\author{Yichao Zhang}
\affiliation{%
  \institution{Tongji University}
  \city{Shanghai}
  \country{China}}

\author{Jinzhi Dong}
\affiliation{%
  \institution{Tongji University}
  \city{Shanghai}
  \country{China}}
  
\author{Shuigeng Zhou}
\orcid{0000-0002-1949-2768}
\affiliation{%
  \institution{Fudan University}
  \city{Shanghai}
  \country{China}}

\renewcommand{\shortauthors}{Anonymous et al.}

\begin{abstract}
Triple set prediction (TSP) aims to discover missing facts from an incomplete knowledge graph without predefined queries. Existing approaches typically construct the predicted set through triple-level scoring, sometimes preceded by entity pair filtering that can exclude valid missing facts. We propose DiffTSP, a conditional discrete diffusion model that formulates missing fact discovery as conditional subgraph generation. In each subgraph, DiffTSP directly predicts the missing relational structure without a separate entity pair selection stage. The model learns to generate masked target relation states at different corruption levels, with the observed support subgraph retained as conditioning context. Our structure-aware denoising network integrates a Relational Context Encoder, which captures local relational patterns through relation-guided message passing, and a Relational Graph Diffusion Transformer, which models broader structural interactions through relational attention. During generation, partially generated relations can also serve as provisional context for subsequent predictions. Experiments on three public datasets demonstrate improved TSP performance over strong baselines. Additional analyses show lower conflict rates in the predicted triple sets relative to the evaluated baselines. Code: https://anonymous.4open.science/r/DiffTSP-0F92/.
\end{abstract}

\maketitle

\section{Introduction}
Knowledge Graphs (KGs), which represent factual information as a collection of \textit{(head entity, relation, tail entity)} triples, have become a cornerstone of modern AI systems~\cite{ji2021survey,liu2024bridgingwww}, powering applications from web search~\cite{wu2024query2gmmwww} to question answering~\cite{zheng2018question,DBLP:conf/naacl/YasunagaRBLL21,kgt5} and recommendation systems~\cite{DBLP:conf/kdd/ZhangJD0YCTHWHC21}. Despite their widespread applications, a fundamental and persistent problem is their inherent incompleteness. 
Real-world KGs usually exhibit extensive incompleteness, with a significant proportion of valid facts missing~\cite{krishnan2024methodwww}.
To solve this problem, Knowledge Graph Completion (KGC) has emerged as an important research topic~\cite{chang2024path,chang2025integrate}.

Existing KGC tasks, such as link prediction~\cite{dettmers2018convolutional,lin2015learning,pan2024unifying} and instance completion~\cite{rosso2021instance}, have made significant strides. However, they operate under a strong and often impractical assumption that some elements of the missing triples are known in advance~\cite{zhang2024starttsp}. For example, link prediction aims to find a missing entity or relation given the other two elements, e.g., \textit{(h, r, ?)} where $h$ is the head entity and $r$ is one certain relation~\cite{shi2024tgformer}. In many real-world scenarios, we lack such prior knowledge. The ultimate goal of automatic KGC is not just to fill in a blank triple by triple, but to discover entirely new and complete factual triples from the incomplete KG. Therefore, a more practical formulation for KGC is the Triple Set Prediction (TSP) task that directly predicts the entire set of missing factual triples, given only the incomplete knowledge graph as input~\cite{yuan2024largetsp}.  The entity set and relation types are given, while the missing combinations of entities and relations must be discovered.
\begin{figure}
    \centering
   \includegraphics[width=0.9\linewidth]{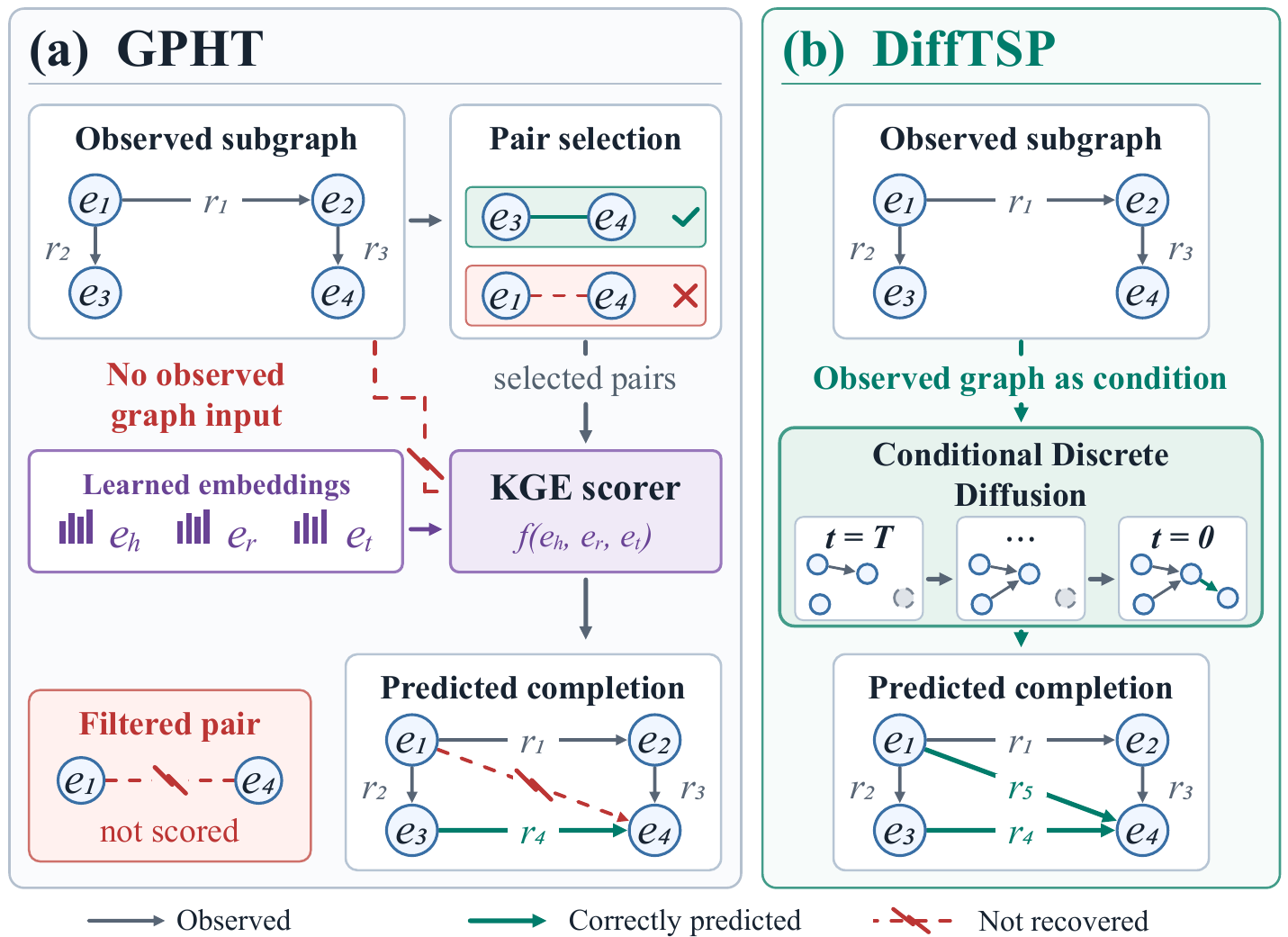}
    \caption{(a) GPHT performs entity pair selection followed by triple-wise scoring, potentially excluding valid missing facts during pair filtering. (b) DiffTSP directly generates missing relational structures conditioned on the observed subgraph.}
    \label{fig:motivation}
\end{figure}

TSP requires discovering missing facts from a large space of candidate triples. Knowledge graph embedding (KGE) methods adapted to this task score candidate triples and select those above a threshold. These methods may assign high scores to triples even when the relation is not applicable to the head entity, leading to spurious predictions. To reduce the search space, GPHT~\cite{zhang2024starttsp} first partitions the knowledge graph into subgraphs, then uses a graph encoder to select entity pairs and a separately trained KGE model to predict their possible relations. This decomposition improves efficiency but introduces a potential recall bottleneck: an entity pair rejected during selection cannot contribute any missing relations to the subsequent prediction stage. Moreover, although GPHT uses graph representations for pair selection, its KGE scorer evaluates relations from learned entity and relation embeddings rather than taking the observed subgraph as an explicit input. These embeddings encode information from training facts, but the relational context is not directly supplied to the relation scorer.

We instead formulate missing fact discovery as conditional subgraph generation and complete each subgraph as a whole. The observed triples provide the conditioning context, and the missing relational structure forms the prediction target. Under this formulation, the locations and types of missing relations are inferred by the same graph completion model, rather than through separate entity pair selection and relation scoring stages. We retain subgraphs as bounded prediction domains, while removing the additional entity pair filtering step within each subgraph. As illustrated in Figure~\ref{fig:motivation}(b), we introduce \M{}, a conditional discrete diffusion model that learns to generate missing relational edges from the observed support subgraph.

Discrete diffusion provides a framework for learning this conditional distribution through graph denoising~\cite{ho2020ddpm,vignac2022digress,liu2024graphdit}. Operating on discrete edge states keeps corruption and reconstruction in the same space as the relational prediction target, while conditioning allows the observed graph to remain fixed throughout completion. Specifically, we represent candidate relations using relational edge states. During training, facts are divided into a support subgraph and reconstruction targets. The target relation states are progressively masked, and the model learns to reconstruct them at different corruption levels with the support subgraph retained as conditioning context. This turns missing triples prediction into a collection of conditional graph generation tasks.


To make generation sensitive to relational structure, we design a denoising network that combines a Relational Context Encoder (RCE) with a Relational Graph Diffusion Transformer (RelDiT). RCE conditions neighborhood aggregation on relation types, producing entity representations that reflect local relational patterns. RelDiT incorporates relational information into attention, allowing broader interactions across the subgraph to inform edge generation. Together, these modules bring both local relational evidence and wider structural context into candidate edge prediction. A weighted reconstruction objective addresses imbalance among edge states and relation types. During generation, the denoiser also receives the partially generated graph, allowing previously generated relations to serve as provisional context for subsequent predictions. 

Our contributions are threefold:
1) We propose a conditional discrete diffusion model for TSP, formulating missing triples discovery as subgraph generation without a separate entity pair selection stage within each subgraph.
2) We design a structure-aware denoising network that combines relation-guided message passing with relational attention, enabling missing relations to be generated using local and broader subgraph context.
3) We evaluate \M{} on three public datasets, demonstrating improved TSP performance over strong baselines and examining the roles of its key components through ablation studies. Additional analyses show lower conflict rates in the predicted triple sets than those of the compared methods.

\section{Related Works}

\subsection{Knowledge Graph Completion}
KGC aims to address the incompleteness issue of KGs by inferring missing facts~\cite{wang2025raker}. This task has been formulated in different ways, e.g., Link Prediction, Instance Completion, and Triple Set Prediction~\cite{zhang2024starttsp}.
Specifically, link prediction is the most widely studied KGC task. Given a triple with one missing element, either the head entity \textit{(?, r, t)}, tail entity \textit{(h, r, ?)}, or relation \textit{(h, ?, t)}, the goal is to predict the missing element. A representative approach for link prediction in KGs is to learn low-dimensional embeddings for entities and relations (e.g., TransE~\cite{bordes2013translating}, RotatE~\cite{sun2019rotate}, HAKE~\cite{zhang2020hake}, PairE~\cite{chao2020pairre}), and use a scoring function to rank candidate entities or relations. 
There are also some methods using graph networks (e.g., RGCN~\cite{2018rgcn}, CompGCN~\cite{CompGCN},  MorsE~\cite{DBLP:conf/sigir/Chen0ZZYXC22}, AstarNet~\cite{DBLP:conf/nips/astarnet}, and ULTRA~\cite{DBLP:conf/iclr/ultra} ).
These methods are primarily developed for link prediction, where part of the target triple is given and candidate entities are ranked.  
Instance completion~\cite{rosso2021instance} predicts all missing relation-tail pairs \textit{(?, ?)} for a given head entity \textit{h}. This formulation moves from single-element prediction to a one-to-many completion setting, but still requires the head entity to be specified beforehand.
Triple Set Prediction~\cite{yuan2024largetsp,zhang2024starttsp} is the most general and challenging formulation for the KGC task, and aims to predict the whole set of missing factual triples given only the existing incomplete graph. GPHT combines graph partitioning with entity pair selection and subsequent KGE-based relation scoring. Our approach instead treats each subgraph as a whole and formulates missing fact discovery as conditional graph generation. 

\subsection{Diffusion Models for Graph Generation}
Diffusion models~\cite{ho2020ddpm} are a type of powerful deep generative models, achieving excellent performance in image~\cite{rombach2022latent} and language generation~\cite{nie2025llada,arriola2025block}. They progressively add noise to data in a forward process, and then learn a neural network to reverse this process, i.e., starting from pure noise to generate new data samples.
Applying diffusion models to discrete and structured data like graphs presents unique challenges. Some recent studies~\cite{jo2022gdoss,gao2026towardyangqiu,fu2024hyperbolicyangqiu} have made significant progress in this direction.
DiGress~\cite{vignac2022digress} defines a new discrete diffusion process over graphs. GraphDiT~\cite{liu2024graphdit} further advanced this line of research by incorporating the Dit~\cite{peebles2023dit} architecture into the denoising network, demonstrating that attention-based mechanisms are crucial for modeling graph properties in a diffusion model.
However, a limitation of existing graph generation models is that they are designed for generating generic graphs. Generic graph diffusion typically treats the edge between a node pair as one categorical variable~\cite{fu2026graphyangqiu}, whereas TSP requires prediction over multiple relation types. Furthermore, existing graph diffusion methods primarily focus on generating a complete graph conditioned on simple class attributes. In contrast, our work uses an incomplete KG as condition for the diffusion process.  

\section{Preliminaries}
Each KG is a directed multi-relational graph $G = (\mathcal V, \mathcal R, \mathcal T)$, where $\mathcal V$ is the set of entities, $\mathcal R$ is the set of relations, and $\mathcal T$ is the set of factual triples. Each triple $(h, r, t) \in \mathcal T$ consists of a head entity $h \in \mathcal V$, a relation $r \in \mathcal R$, and a tail entity $t \in \mathcal V$, representing a link from $h$ to $t$ via $r$. 
Formally, given an incomplete KG $G_{train} = (\mathcal V, \mathcal R, \mathcal T_{train})$ as input, the goal of the TSP task is to predict the set of missing triples $\mathcal T_{pred}$, ($\mathcal T_{pred} \cap \mathcal T_{train} = \emptyset$). We use $\mathcal T_{test}$ as the set of true triples in the test set to evaluate the model. 

\section{Methodology}
\begin{figure*}
    \centering
    \includegraphics[width=1\linewidth]{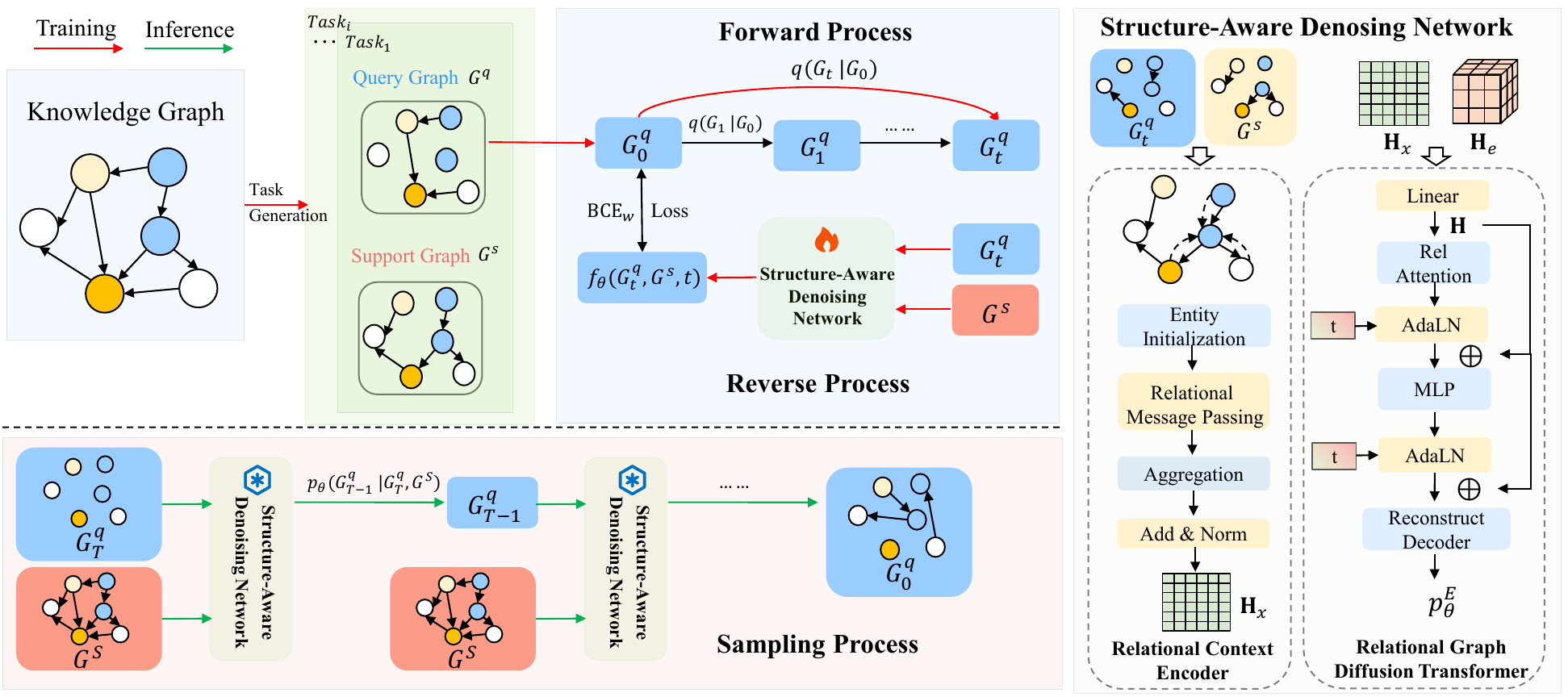}
    \caption{The training process and inference process of \M. Given a support graph and a noisy query graph, DiffTSP first corrupts the query graph through a discrete forward diffusion process, and then learns a structure-aware denoising network to recover the missing triples. The denoising network introduces a Relational Context Encoder and a Relational Graph Diffusion Transformer. During inference, DiffTSP iteratively denoises the query graph and generates the final triple set.}
    \label{fig:model}
\end{figure*}
\subsection{Overview of DiffTSP}
As shown in Figure 2, \M{} constructs support-query tasks and learns to reconstruct corrupted the query graph conditioned on the support graph. Its structure-aware denoiser combines RCE and RelDiT. At inference, the model starts from an empty query graph and iteratively generates the missing triple set.

\subsection{Relation-Balanced Learning Task Generation}
We employ a support-query task generation scheme to explicitly train \M~to generate query triple sets from support sets. 
For the given graph $G = (\mathcal V_{}, \mathcal R_{}, \mathcal T_{})$, multiple training tasks are generated. For each task, the set of triples $\mathcal T_{}$ is partitioned into a support set $\mathcal T^s$ and a query set $\mathcal T^q$. We use parameter $\rho \in (0,1)$ to specify the proportion of triples allocated to $\mathcal T^s$. This yields a support graph $G^s = (\mathcal V_{}, \mathcal R_{}, \mathcal T^s)$ and a query graph $G^q = (\mathcal V_{}, \mathcal R_{}, \mathcal T^q)$.
To mitigate the bias towards frequently occurring relations and ensure balance, we design a relation-balanced split strategy. It ensures that the probability distribution of relations in $\mathcal{T}^s$ and $\mathcal T^q$ mirrors that of the original graph $\mathcal T_{}$.
Let $P(r | \mathcal T)$ represent the probability of a relation $r$ in a set of triples $\mathcal T$. For any relation $r \in \mathcal R_{}$, we have:
\begin{equation}
    P(r | \mathcal T^s) \approx P(r | \mathcal T^q) \approx P(r | \mathcal T).
\end{equation}
This process is repeated $N_s$ times for $G_{}$ to generate a collection of $(G^s, G^q)$ pairs. This ensures that \M~is trained in a variety of combinations of support and query graphs, thereby achieving robust graph generalization.

\subsection{Forward Diffusion Process}
We add structured noise into the query graph $G^q$ and subsequently perform denoising to recover the original structure. The support graph $G^s$ serves as a condition during the denoising phase. 
For clarity in the following discussion, we denote the relations among $n$ entities as an adjacency tensor $E^q \in \mathbb{R}^{n \times n \times b \times 2}$, where $b=|\mathcal R|+1$ counts the KG relation types and one additional no-edge type. The first $b-1$ channels correspond to KG relations, while the last channel corresponds to no-edge type. The forward diffusion transitions apply only to the KG relation channels $k<b$, excluding the no-edge channel. For each relation type $k$ between entities $i$ and $j$, the edge $E^q_{ijk}$ is a one-hot vector that indicates its existence ($[1,0]$) or absence ($[0,1]$).  

The goal of the TSP task is to complete the KG by inferring missing edges between a fixed set of entities. This allows the diffusion process on the graph $G^q$ to be simplified to operate directly on its adjacency tensor $E^q$. The entity and relation vocabularies are fixed by the TSP setting.
The forward diffusion process is a fixed Markov chain that incrementally adds noise to the adjacency tensor $E^q_0$($=E^q$). The transition at each step $t$ is defined by a categorical distribution (Cat):
\begin{equation}
    q(G^q_t | G^q_{t-1})=q(E^q_t | E^q_{t-1}) = \text{Cat}(E^q_t; E^q_{t-1} Q_t),
\end{equation}

\noindent Here, $E_t^q$ is the noisy adjacency tensor at timestep $t$. The operation $E_{t-1}^q Q_t$ involves an independent matrix multiplication for each potential edge to update its state probability. 
$Q_t \in \mathbb{R}^{2 \times2}$ defines the transition probabilities of a relational edge from timestep $t-1$ to timestep $t$. The entry $[Q_t]_{i,j}$ specifies the probability of transitioning from state $i$ to $j$ for a specific edge. Denoting the absence of an edge as state M, the transition matrix $Q_t$ is defined as:
\begin{equation}
\label{eq:qitem}
    [Q_t]_{i,j} = \begin{cases} 1 & \text{if } i = j = \text{M}, \\ 1-\alpha_{t|t-1} & \text{if } j = \text{M}, i \neq \text{M}, \\ \alpha_{t|t-1} & \text{if } i = j \neq \text{M}, \\ 0 & \text{otherwise}. \end{cases}
\end{equation}

\noindent where the noise schedule parameter $\alpha_{t|t-1} \in (0, 1)$ is a function decreasing with $t$. This formulation ensures that, at each step, a relational edge either remains its current state (with probability $\alpha_{t|t-1}$) or transitions to the absence state (with probability $1-\alpha_{t|t-1}$).
Actually, we can sample $E_t^q$ directly from $E_0^q$:
\begin{equation}
    q(G_t^q|G_0^q) = \text{Cat}(E_t^q; E_0^q \bar{Q}_t), \quad \text{with} \quad \bar{Q}_t = \prod_{i=1}^{t} Q_i
\end{equation}

\begin{equation}
    [\bar{Q}_t]_{i,j} = \begin{cases} 1 & \text{if } i = j = \text{M}, \\ 1-\alpha_{t} & \text{if } j = \text{M}, i \neq \text{M}, \\ \alpha_{t} & \text{if } i = j \neq \text{M}, \\ 0 & \text{otherwise}. \end{cases}
\end{equation}

\noindent where $\alpha_t=\prod_{i=1}^t \alpha_{i|i-1}$.
When $t$ approaches infinity, $G^q_t$ will only have entity nodes and no edges. The support graph remains available as a condition. 

\subsection{Structure-Aware Denoising Network}
Our structure-aware denoising network $f_\theta(G_t^q, G^s, t)$, takes the noisy query graph $G_t^q$, the support graph $G^s$, and the timestep $t$ as input. As depicted in Figure~\ref{fig:model},  our denoising network consists of a relational context encoder and a relational graph diffusion transformer, which capture local relational structures and long-range structural dependencies, respectively.
To effectively leverage the structural information from the support graph, we fuse $G^s$ with the query graph $G_t^q$. 
The fusion is performed by taking the union of their relational edge sets. In this case, an edge is included in the fused graph if it appears in either $G^s$ or $G_t^q$. The fused graph is then used as the input of the denoising network.
In the RCE module, we maintain a learnable relation embedding matrix $\mathbf{R} \in \mathbb{R}^{2n_r \times a}$, where $n_r$ is the number of relation types and $a$ is the embedding dimension. This matrix stores embeddings for each relation type and its corresponding inverse, allowing the model to differentiate edge directions. The initial feature vector $\mathbf{h}_i^{(0)}$ for an entity $i$ is computed by averaging the embeddings of all its adjacent relations:
\begin{small}
   \begin{equation}
 \mathbf{h}_i^{(0)} = \frac{1}{d_i} \left( \sum_{(i, r, j) \in {\mathcal T}} \mathbf{R}_r + \sum_{(j, r, i) \in {\mathcal T}} \mathbf{R}_{r^{-1}} \right).
\end{equation} 
\end{small}

\noindent Here, $\mathcal{T}$ represents the triples in the graph, $d_i$ is the total degree of entity $i$, and $\mathbf{R}_r$ and $\mathbf{R}_{r^{-1}}$ are the learnable embeddings for relation $r$ and its inverse, respectively. 
The entity representation $\mathbf{h}_i^{(l+1)}$ for entity $i$ at layer $l+1$ is updated as below.
\begin{equation}
    \mathbf{h}_i^{(l+1)} = \sigma \left( \sum_{r \in R} \sum_{j \in N_i^r} \frac{1}{c_{i,r}} W_r^{(l)} \mathbf{h}_j^{(l)} + W_0^{(l)} \mathbf{h}_i^{(l)} \right),
\end{equation}

\noindent where $\sigma$ is an activation function, $N_i^r$ is the set of neighbors of entity $i$ under relation $r$, $W_r^{(l)}$ and $W_0^{(l)}$ are learnable weights at layer $l$, and $c_{i,r}$ is a normalization constant.

In the RelDiT module, the entity features $\mathbf{H}_x$ are initialized using the final entity representations from the RCE, and the edge features $\mathbf{H}_e$ are initialized using the multi-hot encoding from the fused graph. These features are fused through a linear projection, i.e., $\mathbf{H}^{(0)} = \text{Linear}(Concat(\mathbf{H}_x, \mathbf{H}_e)).$
Following~\cite{liu2024graphdit}, the scalar timestep $t$ is encoded by a sinusoidal embedding $\tau$. In each RelDiT block, the time embedding is injected via adaptive layer normalization (adaLN). Specifically, the $k$-th block is defined as
    
\begin{equation}
\begin{split}
\mathbf H^{(k, \text{attn})} &= \mathbf H^{(k-1)} + \text{AdaLN}(\text{RelAttn}(\mathbf H^{(k-1)}), \tau),\\
\mathbf H^{(k)} &= \mathbf H^{(k, \text{attn})} + \text{AdaLN}(\text{MLP}(\mathbf H^{(k, \text{attn})}), \tau).
\end{split}
\end{equation}
The attention layer RelAttn($\cdot$) operates on the representation $\mathbf H^{(k-1)}$. Let $\mathbf h_i$ and $\mathbf h_j$ denote the representations of entities $i$ and $j$ in $\mathbf H^{(k-1)}$. Unlike standard graph transformers that treat relations only as edge channels, we incorporate relational information into the attention computation through a relation-aware attention bias. Formally, the attention score between entity $i$ and entity $j$ is computed as

\begin{small}
\begin{equation}
\text{RelAttn}(i, j) =\text{Softmax}\left(\frac{(\mathbf h_i \mathbf W_Q)(\mathbf h_j \mathbf W_K)^{T}}{\sqrt{d}}+\mathbf B_{ij}\right),
\end{equation}
\end{small}

\noindent where $\mathbf B_{ij}$ denotes the relation-aware attention bias that reflects the relational dependency between the two entities. Specifically, the bias is defined as $\mathbf B_{ij} =\sum_{k=1}^{b} P_t(i,k,j)\, \mathbf r_k,$ where $P_t(i,k,j) \in [0,1]$ represents the probability that relation type $k$ exists between entities $i$ and $j$ at timestep $t$, $b$ is the total number of relation types, and $\mathbf r_k$ is a learnable embedding associated with relation $k$. 

After stacking multiple RelDiT blocks, the final representation $\mathbf H_{\text{final}}$ provides a purified representation of the latent graph structure. We then apply a reconstruction decoder to predict the probabilities of the clean relational adjacency tensor: $p_\theta^E \in \mathbb{R}^{n \times n \times b}=\text{Linear}(\text{MLP}(\mathbf H_{\text{final}})).$

\subsection{Reverse Process} 

The structure-aware denoising network $f_\theta$ is trained to reverse the forward diffusion process, i.e., predicting the clean graph $G_0^q$ from its noisy version $G_t^q$ under the condition $G^s$. Formally, this corresponds to learning the conditional reverse process $p_\theta(G_{t-1}^q \mid G_t^q, G^s)$ such that samples drawn from the reverse diffusion chain recover the data distribution. 
We design a weighted Binary Cross-Entropy ($\text{BCE}_w$) loss function to assign weights to each relation type based on its inverse frequency, thus effectively mitigating the imbalance caused by the varying counts of other relations. The weights are computed as inverse frequencies in the clean training targets. Each weight is applied only to the positive label term
of the corresponding binary cross entropy loss. The loss is averaged over entries retained after input edge exclusion. 
The final training objective is to optimize the $\text{BCE}_w$ loss between the true adjacency tensor of $G_0^q$ and the prediction by our model $f_\theta(G_t^q,G^s,t)$, i.e.,
We train the denoiser using a weighted reconstruction objective over randomly sampled diffusion timesteps.
\begin{equation}
    \mathcal{L}_{simple}
    = \mathbb{E}_{t, (G_0^q, G^s), G_t^q} 
    \left[
    \text{BCE}_w(G_0^q, f_{\theta}(G_t^q, G^s, t))
    \right].
\end{equation}
\noindent where $e_{ijk} \in \{0,1\}$ is the ground-truth value of the edge of relation type $k$ between entities $i$ and $j$. We denote ${p}_{\theta}^E=f_\theta(G_t^q,G^s,t)\in \mathbb{R}^{n \times n \times b}$, and ${p}_{\theta,ijk}^E \in [0,1]$ is the predicted existence probability for edge $(i,j,k)$. In practice, to ensure that the model learns to generate unknown relational edges in $G_0^q$, we exclude the relational edges in $G^s$ and $G_t^q$ from the loss computation. The variational formulation of the reverse process is discussed in Appendix~\ref{sec:reverseprocess}.

\subsection{Sampling Process}

Although the denoising network is trained to directly predict the final graph $G_0^q$, the sampling process remains iterative, spanning $T$ steps as illustrated in Figure~\ref{fig:model}. The process is initialized with two inputs, i.e., the support graph $G^S$ from the training data, and the noisy query graph $G_T^q$ that contains the same set of entities as $G^S$ but has no edges. Assuming that the transitions of each relational edge in $G^q$ are conditionally independent given the state at time $t$, as formulated in standard discrete diffusion models~\cite{austin2021d3pm,vignac2022digress}, the probability of graph $G^q_{t-1}$ given $G^q_{t}$ and $G^S$ can be expressed as a product over all possible edges:

\begin{small}
\begin{equation}
    p_\theta(G_{t-1}^q | G_t^q,G^S)= \prod_{ i, j ,k} p_\theta(e_{ijk}^{t-1} | G_t^q,G^S),
\end{equation}
\end{small}

\noindent where $p_\theta(e_{ijk}^{t-1} | G_t^q,G^S)$ represents the reverse transition probability for the edge of relation type $k$ between entities $i$ and $j$. To derive this probability, we first consider the true posterior distribution conditioned on the initial state. By Bayes' rule, we have (for brevity, we define $E_t=E_t^q$):

\begin{small}
 \begin{equation}
\label{eq:posterior_bayes}
\begin{split}
    q(E_{t-1}|E_t, E_0) & =\frac{q(E_t|E_{t-1})q(E_{t-1}|E_0)}{q(E_t|E_0)}\\
    &= \text{Cat}\left(E_{t-1}; p = \frac{E_t Q_t^\top \odot E_0 \overline{Q}_{t-1}}{E_0 \overline{Q}_t E_t^\top}\right).
\end{split}
\end{equation}   
\end{small}

\noindent Substituting the formulation for $Q$ into Eq.~(\ref{eq:posterior_bayes}), we obtain the specific transition probabilities for an edge state $e_t$:

\begin{small}
 \begin{equation}
    \begin{split}
        &q(e_{t-1}=1 | e_t=\text{M}, e_0=1) = \frac{\alpha_{t-1}  (1 - \alpha_{t|t-1})}{1 - \alpha_t} = \frac{\alpha_{t-1} - \alpha_{t}}{1 - \alpha_t}, \\
        &q(e_{t-1}=\text{M} | e_t=\text{M}, e_0=1) = \frac{1 - \alpha_{t-1}}{1 - \alpha_t},
    \end{split}
\end{equation}   
\end{small}

\noindent where $e_{t}$ denotes the state of an edge in the timestep $t$. In our practical implementation, $e_t=1$ means the existence of the edge, while $e_t=\text{M}=0$ indicates its absence. This is equivalent to the one-hot encoding of \text{[1,0]} for existence and \text{[0,1]} for absence, as described in the main text.



During sampling, the denoising network predicts the clean edge-existence probabilities $p_\theta^E=f_\theta(G_t^q,G^S,t)$. 
We use the binary representation $\mathrm{M}=0$ and define $\rho_t=(\alpha_{t-1}-\alpha_t)/(1-\alpha_t)$. For a currently absent edge, marginalizing over its unknown clean state gives the base reverse transition:
\begin{equation}
\begin{aligned}
&p_\theta(e_{ijk}^{t-1}=1\mid G_t^q,G^S)\\
&\quad =
\sum_{x\in\{0,1\}}
q(e_{ijk}^{t-1}=1\mid e_{ijk}^{t}=0,e_{ijk}^{0}=x)
p_\theta(e_{ijk}^{0}=x\mid G_t^q,G^S)\\
&\quad = \rho_t p_{\theta,ijk}^{E},
\qquad e_{ijk}^{t}=0.
\end{aligned}
\end{equation}
Previously generated edges are retained.

In practice, we use a thresholded decoding rule based on this
transition. The first $b-1$ output channels correspond to relations, while the last channel provides an auxiliary no-edge type score. Define
\begin{equation}
\widehat y_{ijk}^{t}
=
\mathbf 1\!\left(p_{\theta,ijk}^{E}>\gamma\right)
\mathbf 1\!\left(
\max_{1\leq \ell\leq b-1}p_{\theta,ij\ell}^{E}
>
p_{\theta,ijb}^{E}
\right)
\label{eq:thresholded-decoding}
\end{equation}
For candidate triples outside the support graph, the
implemented sampling probability is
\begin{equation}
\tilde p_\theta(e_{ijk}^{t-1}=1\mid G_t^q,G^S)
=
\begin{cases}
1, & e_{ijk}^{t}=1,\\
\rho_t\hat y_{ijk}^{t}, & e_{ijk}^{t}=0.
\end{cases}
\end{equation}
The probability of state $0$ is the complement.
This thresholded rule~\cite{nie2025llada} is a decoding modification of the base reverse transition. The no-edge decision is recomputed at each step and only suppresses new additions. Support triples are excluded from generation.

The overall diffusion training and sampling process are detailed with algorithms in Appendix~\ref{sec:algorithms}.  

\section{Experiment}

\begin{table*}[t]
    \centering
    \caption{The results on Wiki79k and Wiki143k datasets, where $\mathcal T^{POWA+}_{pred}$ is the correct triple set in test data, and $\mathcal T^{POWA}_{pred}$ is the triple set that can be confidently judged under the RS-POWA assumption. The best results are in bold, and the second-best results are underlined. The results with \textsuperscript{†} are from GPHT~\cite{zhang2024starttsp}.}
    \label{tab:owa}
     \scalebox{0.9}{%
    \begin{tabular}{l|l|ccc|cc>{\columncolor{gray!10}}c|c>{\columncolor{gray!10}}c}
        \toprule
        Datasets & Models & \multicolumn{3}{c|}{Number of Triples in} & \multicolumn{3}{c|}{RS-POWA Metrics} &\multicolumn{2}{c}{CWA Metrics} \\ 
        \midrule
        & & ${\mathcal T}_{pred}$ & ${\mathcal T}^{POWA}_{pred}$ & ${\mathcal T}^{POWA+}_{pred}$ & $JPrecision$ & $STRecall$ & $F_{TSP}$  &$JPrecision$& $F_{TSP}$\\
        \midrule
        \multirow{7}{*}{Wiki79k} & 
         HAKE\textsuperscript{†} & 125$\pm$87 & 119$\pm$83 & 30$\pm$7 & 0.246$\pm$14.8\% & 0.044$\pm$0.4\% & 0.075$\pm$0.3\%& \underline{0.240}&0.074  \\
        & PairRE\textsuperscript{†} & 29742$\pm$105 & 10217$\pm$51 & 2901$\pm$7 & 0.191$\pm$0.1\% & \underline{0.428$\pm$0.0\%} & 0.264$\pm$0.2\% & 0.097&0.158 \\
        & AstarNet & 5689$\pm287$ & 1750$\pm116$ & 1328$\pm52$ & 0.496$\pm$0.8\% & 0.289$\pm$1.1\% & \underline{0.365$\pm$0.2\%}& 0.233&\underline{0.258}  \\
        & ULTRA & 9743$\pm0$ & 2945$\pm0$ & 1062$\pm0$ & 0.235$\pm$0.0\% & 0.258$\pm$0.0\% & 0.246$\pm$0.0\%& 0.109&0.153  \\
         & GPHT(HAKE)\textsuperscript{†} & 209$\pm$47 & 191$\pm$47 & 44$\pm$8 & 0.220$\pm$8.0\% & 0.053$\pm$0.4\% & 0.085$\pm$1.1\% & 0.210&0.084 \\
        & GPHT(PairRE)\textsuperscript{†} & 12392$\pm$3813 & 5866$\pm$1262 & 2018$\pm$332 & \text{0.253$\pm$1.6\%} & 0.357$\pm$2.8\% & \text{0.296$\pm$0.4\%}& 0.162&0.224  \\
        \midrule
        & DiffTSP & 7472$\pm$103 & 4355$\pm$51 & 3472$\pm$87 & \textbf{0.630$\pm$1.2\%} & \textbf{0.468$\pm$0.4\%} & \textbf{0.537$\pm$0.5\%}& \textbf{0.464}&\textbf{0.466}  \\
        \midrule
        \multirow{7}{*}{Wiki143k} & 
        
        HAKE\textsuperscript{†} & 22215$\pm$1283 & 6182$\pm$43 & 3044$\pm$72 & 0.315$\pm$0.2\% & \underline{0.326$\pm$0.3\%} & 0.32$\pm$0.3\%  & 0.137&0.193\\
        & PairRE\textsuperscript{†} & 19228$\pm$2075 & 3313$\pm$722 & 1191$\pm$77 & 0.211$\pm$3.0\% & 0.204$\pm$0.7\% & 0.207$\pm$1.0\% & 0.061&0.095\\
         & AstarNet & 6001$\pm231$ & 2589$\pm129$ & 2275$\pm31$ & \textbf{0.629$\pm$0.8\%} & 0.281$\pm$0.3\% & \underline{0.389$\pm$0.2\%}& \underline{0.379}& \underline{0.323}  \\
        & ULTRA & 3844$\pm0$ & 1349$\pm0$ & 692$\pm0$ & 0.346$\pm$0.0\% & 0.155$\pm$0.0\% & 0.214$\pm$0.0\%& 0.180&0.166  \\
        & GPHT(HAKE)\textsuperscript{†} & 17702$\pm$7935 & 4709$\pm$60 & 2700$\pm$681 & 0.363$\pm$4.8\% & 0.307$\pm$4.2\% & \text{0.333$\pm$4.6\%}  & 0.152&0.204\\
        & GPHT(PairRE)\textsuperscript{†} & 3011$\pm$233 & 1954$\pm$191 & 909$\pm$34 & \text{0.384$\pm$1.9\%} & 0.178$\pm$0.3\% & 0.243$\pm$0.1\% & \text{0.302}&\text{0.224}\\
        \midrule
        & DiffTSP & 10162$\pm$530 & 6804$\pm$275 & 4543$\pm$105 & \underline{0.557$\pm$0.3\%} & \textbf{0.397$\pm$0.4\%} & \textbf{0.464$\pm$0.6\%}& \textbf{0.447}& \textbf{0.420} \\
        \bottomrule
    \end{tabular}}
\end{table*}

\begin{table*}[t]
\caption{The results on CFamily dataset, where $\mathcal T^{CWA+}_{pred}$ is the correct triple set, and $\mathcal T^{CWA}_{pred}=\mathcal T_{pred}$. The results with \textsuperscript{†} are from GPHT~\cite{zhang2024starttsp} and the results with \textsuperscript{*} are from LLMTSP~\cite{yuan2024largetsp}.}
\label{tab:cwa}
\centering
\scalebox{0.9}{
\begin{tabular}{l|ccc|cc>{\columncolor{gray!10}}c}
\toprule
\multicolumn{1}{c|}{Models} & \multicolumn{3}{c|}{Number of Triples in} & \multicolumn{3}{c}{CWA Metrics} \\ \midrule
& $\mathcal T_{pred}$ & $\mathcal T_{pred}^{CWA}$ & $\mathcal T_{pred}^{CWA+}$ & $JPrecision$ & $STRecall$ & $F_{TSP}$ \\ \midrule
HAKE\textsuperscript{†} & 3186$\pm$543 & 3186$\pm$543 & 1788$\pm$321 & 0.561$\pm$1.8\% & 0.624$\pm$5.9\% & \underline{0.591$\pm$3.1\%} \\
PairRE\textsuperscript{†} & 5732$\pm$2062 & 5732$\pm$2062 & 1747$\pm$694 & 0.305$\pm$1.8\% & 0.616$\pm$13.7\% & 0.408$\pm$4.9\% \\
RGCN\textsuperscript{†} & 30608$\pm$3443 & 30608$\pm$3443 & 3026$\pm$137 & 0.099$\pm$6.0\% & \textbf{0.811$\pm$0.5\%} & 0.176$\pm$5.2\% \\
CompGCN\textsuperscript{†} & 38931$\pm$3745 & 38931$\pm$3745 & 2599$\pm$47 & 0.067$\pm$4.8\% & \underline{0.752$\pm$0.3\% } & 0.123$\pm$1.7\% \\ 
AstarNet & 2075$\pm$360 & 2075$\pm$360 & 911$\pm$187 & 0.439$\pm$1.7\% & 0.445$\pm$4.6\% & 0.442$\pm$3.3\% \\
ULTRA & 32761$\pm$0 & 32761$\pm$0 & 763$\pm$0 & 0.023$\pm$0\% & 0.407$\pm$0\% & 0.044$\pm$0\% \\
GPHT(HAKE)\textsuperscript{†} & 1896$\pm$149 & 1896$\pm$149 & 1222$\pm$58 & \underline{0.645$\pm$2.6\%} & 0.516$\pm$1.2\% & 0.573$\pm$0.4\% \\
GPHT(PairRE)\textsuperscript{†} & 3739$\pm$593 & 3739$\pm$593 & 1471$\pm$187 & 0.393$\pm$1.3\% & 0.566$\pm$3.4\% & 0.464$\pm$0.7\% \\
LLMTSP(GPT-3.5)\textsuperscript{*} & 3403$\pm$158 & 3403$\pm$158 & 96$\pm$12 & 0.028$\pm$0.3\% & 0.145$\pm$0.8\% & 0.047$\pm$0.4\% \\
LLMTSP(GPT-4o)\textsuperscript{*} & 1276$\pm$147 & 1276$\pm$147 &179$\pm$17 & 0.140$\pm$0.5\% & 0.197$\pm$0.9\% & 0.164$\pm$0.3\% \\
\midrule
DiffTSP & 2453$\pm$35 & 2453$\pm$35 & 1657$\pm$18 & \textbf{0.675$\pm$0.2\%} & 0.600$\pm$0.3\% & \textbf{0.635$\pm$0.2\%} \\ \bottomrule
\end{tabular}}

\end{table*}

\subsection{Datasets and Evaluation Metrics}
To ensure the reproducibility and fairness of the evaluation, we conduct experiments on three widely recognized benchmark datasets: Wiki79k, Wiki143k, and CFamily, where Wiki79k and Wiki143k are used under the Relation Similarity-based Partial-Open-World Assumption (RS-POWA) and Closed-World Assumption (CWA), and CFamily is used under the CWA. It is worth noting that all these datasets, along with their ground truth, are directly sourced from GPHT~\cite{zhang2024starttsp}. We strictly adhere to the original data partitions and definitions provided by GPHT, which guarantees that our results are directly comparable with existing methods.

To make our approach tractable, we follow the graph partitioning strategy in GPHT~\cite{zhang2024starttsp} to divide the full knowledge graph into a set of smaller, overlapping subgraphs. Our model, \M, is then trained and evaluated on these subgraphs. The final set of predicted triples is the union of the predictions from all subgraphs. Table~\ref{tab:dataset_stats} shows the details of the three datasets.  To ensure a fair and comprehensive comparison, we use the evaluation metrics, i.e., Joint Precision ($JPrecision$), Square-Root Test Recall ($STRecall$), and TSP Score ($F_{TSP}$), proposed for the TSP task in GPHT~\cite{zhang2024starttsp}.
$F_{TSP}$ is the primary comparison metric because it is harmonic mean of $JPrecision$ and $STRecall$. Higher values indicate better outcomes for these metrics. Joint Precision ($JPrecision$) measures the precision of the predicted set while accounting for the size of the set itself. $STRecall$ measures the percentage of test triples successfully recovered by the model. TSP Score ($F_{TSP}$) is the harmonic mean of $JPrecision$ and $STRecall$. 
In CWA, we have:

\begin{small}
\begin{align}
\mathcal {T}^{CWA+}_{{pred}} &= \mathcal{T}_{{pred}} \cap \mathcal {T}_{{test}}, \quad
\mathcal {T}^{CWA-}_{{pred}} = \mathcal {T}_{{pred}} - \mathcal {T}^{CWA+}_{{pred}}, \quad  \nonumber \\
\mathcal {T}^{CWA}_{{pred}} &= \mathcal {T}^{CWA+}_{{pred}} \cup \mathcal {T}^{CWA-}_{{pred}} = \mathcal {T}_{{pred}}
\end{align} 
\end{small}

In RS-POWA, we have:

\begin{small}
\begin{align}
    &\mathcal {T}^{POWA+}_{{pred}}= \mathcal {T}_{{test}} \cap \mathcal {T}_{{pred}}, \quad \mathcal {T}^{POWA}_{{pred}}= \mathcal {T}^{POWA+}_{{pred}} \cup \mathcal {T}^{POWA-}_{{pred}}\\
&\mathcal {T}^{POWA-}_{{pred}} = \{(h, r, t) | (h, r, t) \in \mathcal {T}_{{pred}}, (h, r, t) \notin \mathcal {T}_{{test}}, \nonumber \\
&\exists r' \in \mathcal{R} \ (h, r', t) \in (\mathcal {T}_{{train}} \cup \mathcal {T}_{{test}}) \ \wedge \ {sim}(r, r') < \theta\}
\end{align}
\end{small}

Finally, $JPrecision$, $STRecall$ and $F_{TSP}$ in RS-POWA and CWA are calculated as follows ($WA \in \{POWA, CWA\}$):

\begin{small}
\begin{align}
        JPrecision &= \frac{1}{2}\left(\frac{|\mathcal {T}^{WA+}_{{pred}}|}{|\mathcal {T}^{WA}_{{pred}}|} + \frac{|\mathcal {T}^{WA+}_{{pred}}|}{|\mathcal {T}_{{pred}}|}\right), 
        \quad \nonumber\\STRecall &= \left(\frac{|\mathcal {T}^{WA+}_{{pred}}|}{|\mathcal {T}_{{test}}|}\right)^{\frac{1}{2}}, \quad \nonumber\\
F_{TSP} &= \frac{2 \times (ST{Recall} \times J{Precision})}{ST{Recall} + J{Precision}}
\end{align}
\end{small}

\begin{table}[]
    \centering
      \caption{Statistics of datasets in experiments.}
      
    \label{tab:dataset_stats}
     \scalebox{0.9}{%
    \begin{tabular}{lccccc}
        \toprule
        Datasets & Entity & Relation & Train & Valid & Test \\
        \midrule
        Wiki79k & 7,983 & 85 & 57,033 & 6,337 & 15,843 \\
        Wiki143k & 13,928 & 109  & 103,415 & 11,490 & 28,727  \\
        CFamily & 2,378 & 12  & 16,549 & 1,839 & 4,598  \\
        \bottomrule
    \end{tabular}
    }  
\end{table}

\subsection{Settings}
The implementation is based on PyTorch and trained on one NVIDIA GeForce RTX 4090 GPU.
For our proposed model, hyperparameters are tuned on the validation set. We set the relation-balanced partitioning ratio $\rho=0.8$, repeat times $N_s=100$, embedding dimension $a=16$, sampling steps $T=20$, and the number of Dit blocks $N_{Dit}=3$. We apply a linear noise schedule function $\alpha_t=1-t/T$, and use the Adam optimizer with an initial learning rate of 1e-3. The model is trained for up to 50 epochs, with early stopping triggered if the $F_{TSP}$ score on the validation set does not improve for 10 consecutive epochs.

\subsection{Baselines}
We select both link prediction methods and TSP methods as baselines. The representative link prediction methods, including  HAKE~\cite{zhang2020hake}, PairRE~\cite{chao2020pairre}, RGCN~\cite{2018rgcn}, CompGCN~\cite{CompGCN}, AstarNet~\cite{DBLP:conf/nips/astarnet} and ULTRA~\cite{DBLP:conf/iclr/ultra}, are modified for TSP by scoring all candidate triples in the graph and selecting those above a threshold as the predicted results. The threshold is determined through grid search on the validation set to achieve the best performance. 
The latest TSP methods, such as GPHT~\cite{zhang2024starttsp} and LLMTSP~\cite{yuan2024largetsp}, are also included. For RGCN, CompGCN, and LLMTSP, we use the published TSP results from the corresponding studies, which report these methods only on CFamily. Accordingly, their results are not included on Wiki79k and Wiki143k.
DiGress~\cite{vignac2022digress} and GraphDit~\cite{liu2024graphdit} are generic graph generation models and therefore are not applicable to the TSP protocol without substantial modification. In particular, TSP assumes a fixed entity universe and requires relation-wise conditional completion from an observed support graph, whereas these models are originally designed to generate complete graphs with categorical edge states and class conditions. A direct adaptation would therefore require changing both the generation state space and the conditioning mechanism, making it difficult to interpret as an evaluation of the original methods. 
\subsection{Main Results}

According to the main results in Tables~\ref{tab:owa} and~\ref{tab:cwa},
\M~outperforms all baselines in terms of $F_{TSP}$. Our results are based on 3 runs with different seeds, following the baselines’ setup for fair comparison. To address inconsistencies in prior TSP results for RGCN, CompGCN, and LLMTSP on CFamily, we recomputed the relevant metrics from the reported mean triple counts.
Concretely, in Wiki79k and Wiki143k, we calculate the RS-POWA and CWA metrics, where CWA is calculated using the average value of $\mathcal T_{pred}$ and $\mathcal T_{pred}^{POWA+}$. \M~ achieves SOTA performance in terms of both RS-POWA and CWA metrics. 
In CFamily, our model also achieves SOTA performance on $JPrecision$ and $F_{TSP}$. Though RGCN shows higher $STRecall$, its low $JPrecision$ results in a poor $F_{TSP}$ score.  We observe that the traditional knowledge graph embedding method HAKE achieves the second-best results in terms of $F_{TSP}$, which could be attributed to the relative simplicity of the CFamily dataset, i.e., most of the head and tail entities in the triples to be predicted are in two hops.

\begin{table}[t]
\caption{
Ablation studies.
$F_{TSP}$ in CFamily and Wiki79k are evaluated under CWA and RS-POWA, respectively.}
\label{tab:ablation}
\centering
\small
\begin{tabular}{
    @{}l
    ccc
    @{}
}
\toprule
Variant & CFamily & Wiki79k & Wiki143k \\
\midrule
\multicolumn{4}{@{}l}{\textit{Completion framework}} \\
One-step graph completion
& 0.609 & 0.493 & 0.433 \\
Iterative completion
& 0.623 & 0.508  & 0.441 \\

w/o generated edge context
& 0.585 & 0.466 &  0.420 \\

\midrule
\multicolumn{4}{@{}l}{\textit{Structure-aware denoising}} \\
w/o RCE
& 0.612 & 0.525 & 0.451 \\
w/o RelAttn
& 0.616 & 0.516 & 0.437 \\

\midrule
\multicolumn{4}{@{}l}{\textit{Training design}} \\
w/o weighted BCE
& 0.336 & 0.323 & 0.277 \\
w/o input-edge exclusion
& 0.573 & 0.513 & 0.414 \\
w/ unstratified split
& 0.621 & 0.501 & 0.442 \\

\midrule
\textbf{DiffTSP}
& \bf 0.635 & \bf 0.537 & \bf 0.464 \\
\bottomrule
\end{tabular}
\end{table}

\subsection{Ablation Studies}

We conduct ablation studies on CFamily, Wiki79k and Wiki143k, as shown in Table~\ref{tab:ablation}.

\textbf{Completion framework.}
\textit{One-step graph completion} is trained separately with the same backbone, candidate space, and weighted loss, but without diffusion.
It gets lower $F_{TSP}$ than \M, supporting the conditional diffusion completion design. The comparison with non-diffusion iteration indicates that additional network evaluations alone do not explain the improvement. The decline without generated edge context further supports the value of using previously generated relations as provisional structural evidence. Details are in Appendix~\ref{sec:appenablation}.
\textbf{Structure-aware denoising.}
\textit{w/o RCE} replaces local relational features with random entity embeddings, while \textit{w/o RelAttn} replaces relation-aware attention with standard attention.
Both reduce $F_{TSP}$, supporting the contributions of local relational encoding and relation-aware attention to TSP.
\textbf{Training design.}
\textit{w/o weighted BCE} uses unweighted reconstruction loss.
\textit{w/o input-edge exclusion} includes edges already present in $G^s$ and $G_t^q$ in the loss.
\textit{w/ unstratified split} replaces relation-balanced support-query sampling with random sampling at the same split ratio.
Removing loss weighting substantially degrades performance, highlighting the importance of handling target imbalance.
The other two variants also reduce $F_{TSP}$, supporting the benefits of focusing supervision on relations requiring reconstruction and balancing relation types during training-task construction.

\begin{figure}
    \centering
    \includegraphics[width=0.9\linewidth]{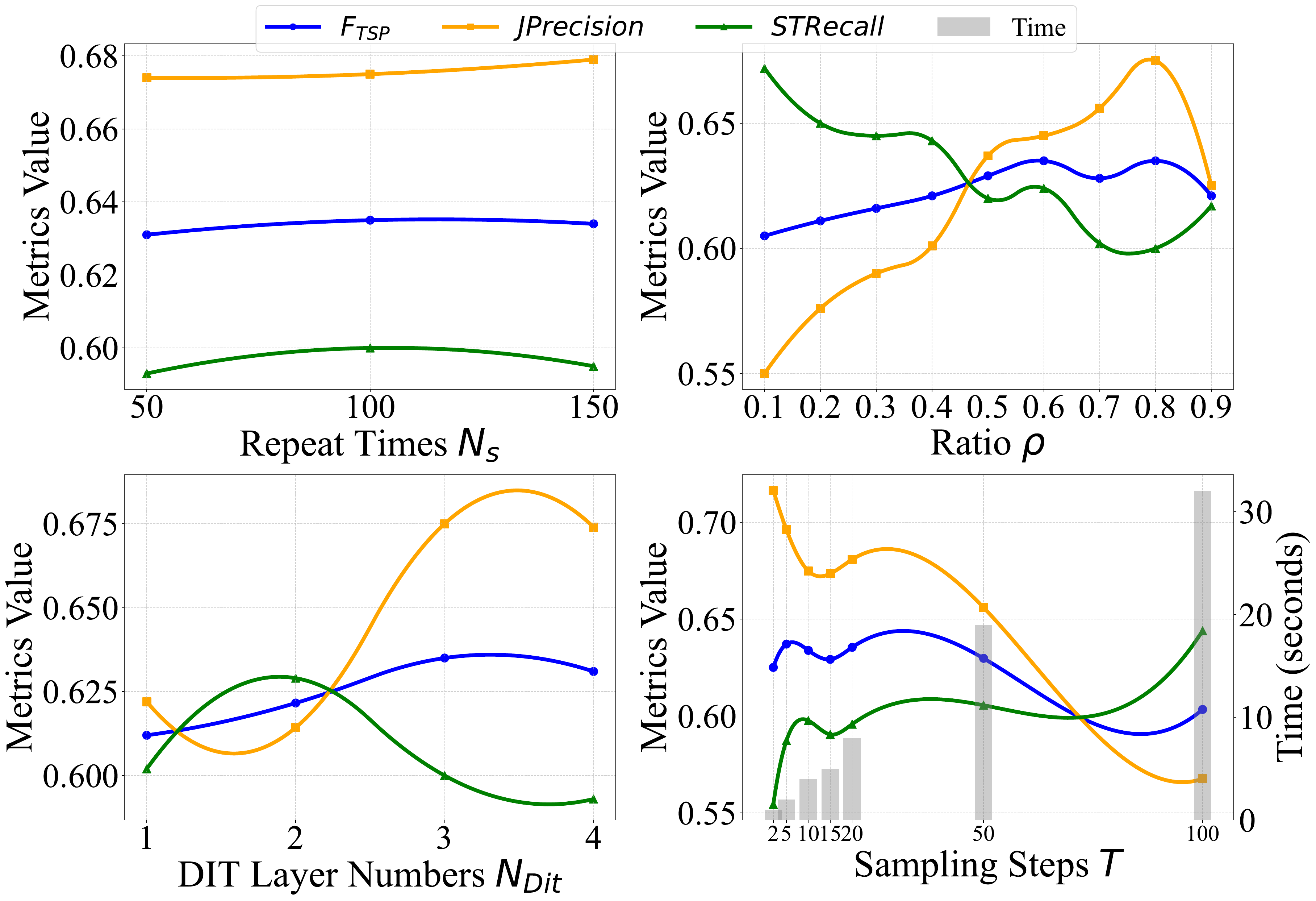}
    \caption{Model performance with different hyperparameter settings.}
    \label{fig:hyper}
\end{figure}
\begin{figure}
    \centering
    \includegraphics[width=1\linewidth]{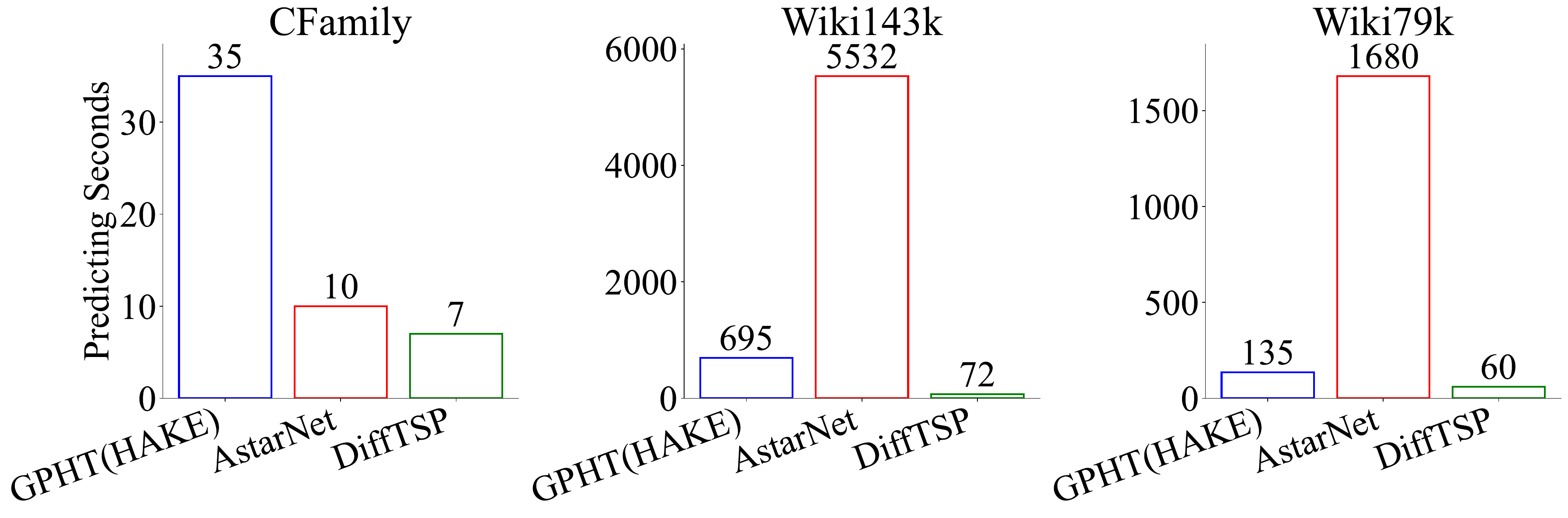}
    \caption{The predicting time of different methods on CFamily, Wiki79k, and Wiki143k.}
    \label{fig:effi}
\end{figure}

\subsection{Hyper-parameter Analysis}

The primary hyperparameters of \M~include the ratio $\rho$ for splitting graphs into support and query graphs, the number of repeat times $N_s$, the number of Dit blocks $N_{Dit}$, and the sampling steps. 
Figure~\ref{fig:hyper} presents the results of \M~while varying these hyperparameters.
We observe that a low ratio $\rho$ leads to a slight decrease in the $F_{TSP}$ score but a relatively higher $STRecall$. This suggests that providing less support encourages the model to predict a broader range of possible triples. Conversely, an excessively high ratio also negatively impacts the $F_{TSP}$ score, likely because it hinders the model's ability to complete missing edges comprehensively.
\M~demonstrates robustness regarding $N_s$ since its performance remains consistent regardless of the value of $N_s$.
For $N_{Dit}$, we observe that increasing the number of blocks generally leads to improved performance. However, to strike a balance between model efficiency and performance, we set $N_{Dit}=3$.
Despite \M~is built on the DDPM framework, it achieves promising results with a small number of steps. This is because the model benefits from condition of the support graph. However, as the number of steps increases, there is a slight decrease in the $F_{TSP}$ metric. We attribute this to the accumulation of erroneous edges during the sampling process. This issue becomes more pronounced with additional steps, leading to a decline in prediction accuracy.

\subsection{Efficiency Analytics}
\label{sec:effi}
Figure~\ref{fig:effi} details the runtime comparison. 
The runtime for these methods is measured under the same experimental conditions. We specifically select GPHT and AstarNet as representative baselines for this comparison. GPHT is chosen as it is the current state-of-the-art method specialized for the TSP task. AstarNet is included as a strong and efficient representative of link prediction models adapted for TSP. The comparison reveals that DiffTSP has a notable efficiency advantage for the TSP task. During inference, the peak GPU memory is $0.5$ GB on CFamily. and $4.3$ / $6.1$ GB on Wiki79k / Wiki143k, comfortably in a single 24 GB GPU.

\subsection{Prediction Conflict Analysis}
\label{sec:setdiagnostics}

Beyond TSP performance, we examine the compatibility of relation combinations in the predicted triple sets. To obtain a common assessment across datasets with different relation vocabularies, we use co-occurrence-based diagnostics that do not require manually specified domain constraints. Pair Conflict Rate (PCR) evaluates combinations of directed relations on the same entity pair, while Role Conflict Rate (RCR) evaluates combinations of relation-specific head and tail roles assigned to the same entity. Both flag potentially incompatible combinations using co-occurrence statistics derived from the training graph.

CFamily's kinship semantics additionally support direct, domain-specific checks. We report Gender Role Conflict Rate (GCR), Parent Cardinality Violation Rate (PCVR), and Kinship Category Conflict Rate (KCR). Under the dataset's kinship conventions, these check for incompatible gendered roles, children assigned at least three distinct parents, and incompatible kinship categories on the same entity pair, respectively.

Each method uses the prediction threshold that maximizes validation $F_{\mathrm{TSP}}$ under the main evaluation protocol. All diagnostics are computed on the resulting final prediction sets. Full metric definitions and counting conventions are provided in Appendix~\ref{sec:consistency_details}.
Table~\ref{tab:conflict_combined} shows that DiffTSP achieves the lowest PCR and RCR among the evaluated methods on all three datasets. On CFamily, it also achieves the lowest GCR, KCR and PCVR. These results show lower conflict rates in DiffTSP's outputs under the evaluated diagnostics, providing an additional empirical benefit alongside its TSP performance.

\begin{table}[t]
\centering
\caption{Conflict and violation rates (\%) comparison. }
\label{tab:conflict_combined}
\small
\begin{tabular}{llrrrrr}
\toprule
Dataset & Method & PCR $\downarrow$ & RCR $\downarrow$ & GCR $\downarrow$ & PCVR $\downarrow$ & KCR $\downarrow$ \\
\midrule
 \multirow{4}{*}{CFamily} & HAKE    & 26.935 & 0.997  & 7.792  & 5.556  & 4.785 \\
        & PairRE  & 47.346 & 15.734 & 60.283 & 2.214  & 8.329 \\
        & GPHT    & 20.937 & 12.727 & 53.266 & \textbf{0.000} & 0.169 \\
        & DiffTSP & \textbf{3.893} & \textbf{0.886} & \textbf{5.569} & \textbf{0.000} & \textbf{0.000} \\
\midrule
\multirow{4}{*}{Wiki79k} & HAKE    & 95.706 & 67.891 & -- & -- & -- \\
        & PairRE  & 40.676 & 15.108 & -- & -- & -- \\
        & GPHT    & 26.241 & 10.615 & -- & -- & -- \\
        & DiffTSP & \textbf{9.722} & \textbf{6.165} & -- & -- & -- \\
\midrule
\multirow{4}{*}{Wiki143k} & HAKE     & 99.353 & 81.522 & -- & -- & -- \\
         & PairRE   & 41.239 & 5.489  & -- & -- & -- \\
         & GPHT & 92.077 & 24.466 & -- & -- & -- \\
         & DiffTSP  & \textbf{8.681} & \textbf{1.789} & -- & -- & -- \\
\bottomrule
\end{tabular}
\vspace{-3mm}
\end{table}

\section{Other Studies}
We provide GPHT's entity pair filtering analysis,  an evaluation of our generative framework in Appendix~\ref{sec:experimentdetail}.

\section{Conclusion}
We presented DiffTSP, a conditional discrete diffusion model that formulates triple set prediction as conditional subgraph generation. In each subgraph, DiffTSP generates missing relational structure conditioned on observed facts. Its structure-aware denoising network combines relation-guided message passing with relational attention to incorporate local and broader structural context. Experiments on three public datasets demonstrate improved TSP performance over strong baselines, with additional analyses showing lower conflict rates in the predicted triple sets.
\bibliographystyle{ACM-Reference-Format}
\bibliography{sample-base}

\appendix
\section{Method Details}
\label{sec:apena}

\subsection{Reverse Process}
\label{sec:reverseprocess}
\subsubsection{Variational formulation of the reverse process}
To learn the reverse process, we aim to maximize the conditional log-likelihood of the data, $\log p_{\theta}(G_0^q | G^s)$. However, directly optimizing this likelihood is intractable due to the marginalization over all latent diffusion states. Following the standard diffusion formulation~\cite{ho2020ddpm} and inspired by recent works~\cite{sohn2015condition,vignac2022digress,ko2024stochasticdvlb}, the conditional reverse model admits a variational lower bound on the log-likelihood. We present this bound to characterize the reverse process and its connection to reconstruction. In practice, optimization uses the weighted reconstruction objective.
Specifically, the objective is derived as follows:
\begin{align}
    &\log p_{\theta}(G_0^q | G^s) 
    = \log \int p_{\theta}(G_{0:T}^q | G^s)\, dG_{1:T}^q \label{eq1} \\
    &= \log \int q(G_{1:T}^q | G_0^q)\,
    \frac{p_{\theta}(G_{0:T}^q | G^s)}{q(G_{1:T}^q | G_0^q)}\, dG_{1:T}^q \label{eq2} \\
    &\geq \mathbb{E}_{q(G_{1:T}^q | G_0^q)} 
    \left[ 
    \log \frac{p_{\theta}(G_{0:T}^q | G^s)}{q(G_{1:T}^q | G_0^q)} 
    \right] 
    \label{eq}
    \quad \text{(denoted as $\mathcal{L}$)} 
\end{align}

The inequality follows from Jensen's inequality and yields a variational lower bound on the log-likelihood.

\subsubsection{Decomposition of the variational lower bound}
Taking expectation over the data distribution, the variational objective becomes
\begin{align}
    &\mathbb{E}_{q}[\mathcal{L}] 
    = \mathbb{E}_q \left[ \log p_{\theta}(G_{0:T}^q | G^s) - \log q(G_{1:T}^q | G_0^q) \right] \label{eq:loss1}.
\end{align}

Using the Markov factorization of the reverse and forward diffusion processes, and Bayes' rule to each forward transition term $q(G_t^q | G_{t-1}^q)$, we have

\begin{small}
\begin{align}
    &\mathbb{E}_{q}[\mathcal{L}] \nonumber \\
    &= \mathbb{E}_q \left[ \log p(G_T^q | G^s) + \sum_{t=1}^T \log p_{\theta}(G_{t-1}^q | G_t^q, G^s) \right. \nonumber\\
    & \quad \left. - \sum_{t=1}^T \log q(G_t^q | G_{t-1}^q) \right] \label{eq:loss2} \\
    &= \mathbb{E}_q \left[ \log p(G_T^q | G^s) + \sum_{t=1}^T \log p_{\theta}(G_{t-1}^q | G_t^q, G^s) \right.  \nonumber \\
    &\quad \left. - \sum_{t=1}^T \log q(G_{t-1}^q | G_t^q, G_0^q) - \log q(G_T^q | G_0^q) \right]\label{eq:loss3} \\
    &= \mathbb{E}_q \left[ \log p_{\theta}(G_0^q | G_1^q, G^s) \right] \nonumber \\
    &\quad + \mathbb{E}_q \left[ \sum_{t=2}^T \left( \log p_{\theta}(G_{t-1}^q | G_t^q, G^s) - \log q(G_{t-1}^q | G_t^q, G_0^q) \right) \right] \nonumber \\
    &\quad + \mathbb{E}_q \left[ \log p(G_T^q | G^s) - \log q(G_T^q | G_0^q) \right] \label{eq:loss4} \\
    &= \mathbb{E}_{q}[\log p_\theta(G_0^q|{G}_1^q,G^s)] \nonumber\\  
    &\quad - \sum_{t=2}^{T} \mathbb{E}_{q} [D_{KL}(q(G_{t-1}^q|G_t^q, G_0^q) || p_\theta(G_{t-1}^q|G_t^q,G^s))] \nonumber \\ 
    &\quad - D_{KL}(q(G_T^q|G_0^q) || p(G_T^q|G^s)) \label{eq:loss5}
\end{align}
\end{small}

Consequently, the negative log-likelihood is upper bounded by the negative variational lower bound:
\begin{align}
   -\mathbb{E}_{q}[\log p_{\theta}(G_0^q | G^s)] 
   \le -\mathbb{E}_{q}[\mathcal{L}]
   =: \mathcal{L}_{VLB}.
\end{align}

\subsection{Training and Sampling Algorithms}
\label{sec:algorithms}
The overall diffusion training and sampling processes are shown in Algorithm~\ref{alg:difftraining} and Algorithm~\ref{alg:diffsample}, respectively.
To handle the sparse and imbalanced relation types of KGs, we treat non-existent edges between two entities as an additional relation type. The non-existent edge type is never noised in the forward diffusion process because there is nothing to mask. For the probability distribution $p_{\theta,ij} ^E \in \mathbb{R}^b$ over relations between entities $i$ and $j$, we first check if the maximum probability corresponds to the ``non-existent" type. If it does, we set the other $b-1$ relation types to zero. Otherwise, following LLaDA~\cite{nie2025llada}, we proceed with the sampling process based on the posterior distribution. 
In the experiments, we have sampling steps $T=20$.

\begin{algorithm}
\caption{Training Process of DiffTSP}
\textbf{Input:} A query graph $G^q = ({X^q}, {E^q})$ and a support graph $G^s$ \Comment{$X^q$ is entity feature, $E^q$ is adjacency tensor}
\begin{algorithmic}[1]
    \State Sample $t \sim \mathcal{U}(1, \dots, T)$ 
    \State Sample $G^q_t \sim {E^q}\bar{Q}^t$ \Comment{Sample a noisy graph}
    \State $ {p}^{E}_\theta \gets f_{\theta}(G^q_t, G^s,t)$ 
    \State optimizer.step($\mathcal{L}_{simple}( G^q,{p}^{E}_\theta)$) \Comment{$\text{BCE}_w$ loss}
\end{algorithmic}
\label{alg:difftraining}
\end{algorithm}

\begin{algorithm}
\caption{Sampling Process of DiffTSP}
\label{alg:diffsample}
\textbf{Input:} Observed support graph
$G^S=(\mathcal V^S,\mathcal R^S,\mathcal T^S)$,
sampling threshold $\gamma$, sampling steps $T$
\begin{algorithmic}[1]
    \State Initialize
    $G_T^q \gets (\mathcal V^S,\mathcal R^S,\emptyset)$

    \For{$t=T,T-1,\ldots,1$}
        \State $p_\theta^E \gets f_\theta(G_t^q,G^S,t)$
        \State $\rho_t \gets
        \frac{\alpha_{t-1}-\alpha_t}{1-\alpha_t}$
        \State $E^{t-1} \gets E^t$
        \Comment{Retain previously generated edges}

        \ForAll{$(i,j)\in\{1,\ldots,n\}^2$}
            \If{$\max_{1\leq k<b}p_{\theta,ijk}^E
                 > p_{\theta,ijb}^E$}
                \For{$k=1,\ldots,b-1$}
                    \If{$e_{ijk}^t=0$ and
                    $(v_i,r_k,v_j)\notin\mathcal T^S$}
                        \State $z_{ijk}\sim
                        \operatorname{Bernoulli}(\rho_t)$
                        \State $e_{ijk}^{t-1}\gets
                        z_{ijk}\,
                        \mathbf{1}(p_{\theta,ijk}^E>\gamma)$
                    \EndIf
                \EndFor
            \EndIf
        \EndFor

        \State Construct $G_{t-1}^q$ from $E^{t-1}$
    \EndFor
    \State \Return $G_0^q$
\end{algorithmic}
\end{algorithm}

\section{Experiment Details}
\label{sec:experimentdetail}
\subsection{Ablation Study Details}
\label{sec:appenablation}
\paragraph{Completion framework.}
Besides one-step completion, we evaluate a non-diffusion iterative completion model with the same backbone, support--query tasks, candidate space, and loss weights. It uses shared parameters across $T=20$ inference passes, matching the number of denoiser evaluations in DiffTSP. At each pass, it adds relations selected by thresholding, then conditions the next pass on the enlarged graph. The iteration index replaces the diffusion timestep in the embedding interface. The model is trained with stage-wise weighted reconstruction supervision on its own detached rollout states, without forward diffusion or
posterior-based sampling.

\paragraph{Generated edge context.}
We further evaluate \emph{DiffTSP w/o generated edge context}. This variant is retrained with the query-state input removed: node features, edge features, and relational attention biases are computed from the support graph only. The query state is maintained solely for loss masking, retaining existing predictions, and selecting eligible sampling positions; it is not provided to the denoiser. The variant otherwise retains the same backbone, diffusion schedule, and $T=20$ sampling steps. Prediction thresholds are selected separately for each variant using the same validation protocol.

\subsection{Subgraph Construction and Prediction Merging}
\label{sec:graphpartition}

\paragraph{Subgraph construction.}
Given the observed KG $G=(\mathcal{V},\mathcal{R},\mathcal{T}_{\mathrm{obs}})$, we adapt the overlapping subgraph construction strategy of GPHT.
Entity groups are formed through multi-hop neighborhood expansion around seed entities, with entities allowed to appear in multiple groups. Each subgraph $G_i=(\mathcal{V}_i,\mathcal{R},\mathcal{T}_i)$ contains the observed triples whose endpoints both belong to $\mathcal{V}_i$. Subgraph construction uses only the observed graph. Prediction is restricted to entity pairs that co-occur in at least one subgraph, and subgraph sizes are capped at 500 entities on Wiki79k and Wiki143k.
Although subgraph construction makes prediction computationally tractable, it can limit recall: valid triples cannot be predicted if their endpoints do not co-occur in any subgraph. Our focus is on conditional completion within the constructed subgraphs without an additional entity pair filtering stage, rather than on optimizing subgraph construction itself. We leave improving candidate coverage while maintaining computational efficiency to future work.

\paragraph{Prediction merging.}
For each subgraph, we obtain a predicted triple set
$\widehat{\mathcal{T}}_i$ using the validation-selected
prediction threshold. We map local entity indices back
to global identifiers and merge the predictions as
\begin{equation}
\widehat{\mathcal{T}}
=
\left(
\bigcup_{i=1}^{m}\widehat{\mathcal{T}}_i
\right)
\setminus \mathcal{T}_{\mathrm{obs}}.
\end{equation}
Identical triples predicted in multiple subgraphs are
retained only once. When a confidence score is required,
we use the highest score assigned to that triple across subgraphs. A prediction retained in one subgraph is not discarded merely because another overlapping subgraph does not predict it.
Evaluation is performed on the final merged triple set.

\paragraph{Subgraph statistics.}
Table~\ref{tab:partition_scope} summarizes the resulting
subgraphs. Overlap measures the percentage of entities
in the full vocabulary that occur in at least two
subgraphs.

\begin{table}[t]
\centering
\caption{Subgraph statistics. Median and max refer to
the number of entities per subgraph. Overlap is the
percentage of vocabulary entities appearing in
multiple subgraphs.}
\label{tab:partition_scope}
\small
\setlength{\tabcolsep}{4pt}
\begin{tabular}{lrrrr}
\toprule
Dataset & Subgraphs & Median & Max & Overlap (\%) \\
\midrule
CFamily  & 75  & 50  & 163 & 40.03 \\
Wiki79k  & 429 & 60  & 500 & 85.36 \\
Wiki143k & 397 & 149 & 500 & 89.75 \\
\bottomrule
\end{tabular}
\end{table}

\paragraph{Storage cost.}
For a subgraph with $n$ entities, a dense adjacency tensor with one FP32 value per candidate relation requires $4n^2|\mathcal{R}|$ bytes. The subgraph-size cap bounds this per-tensor storage cost. Intermediate activations and decoder states require additional GPU memory. Our current study focuses on conditional completion within these subgraphs. Evaluating \M{} on substantially larger KGs, while improving computational efficiency and candidate coverage, is an important direction for future work.

\subsection{Details of Prediction Conflict Analysis}
\label{sec:consistency_details}

\paragraph{Evaluation protocol.}
Let $\widehat{\mathcal T}$ denote the final prediction set obtained at the validation-selected threshold. All diagnostics are computed in $\widehat{\mathcal T}$, using the same rules across methods. Rates are reported as percentages.

\paragraph{Co-occurrence-based diagnostics.}
For PCR, a counting unit is an entity pair, and its labels are the relations on that pair with directions preserved. For RCR, a unit is an entity, and its labels are relation roles $(r,\mathrm{head})$ or $(r,\mathrm{tail})$. On the training graph, let $N$ be the number of units, $n_a$ the number carrying label $a$, and $n_{ab}$ the number carrying both $a$ and $b$. The expected co-occurrence under independence is
$\bar n_{ab}=\frac{n_a n_b}{N}.$
For diagnostic $d\in\{P,R\}$, we identify potentially incompatible label pairs as
\begin{equation}
\mathcal B_d
=
\bigl\{\{a,b\}:a\neq b,\;
n_{ab}<\tau_d\bar n_{ab}\bigr\},
\qquad 0<\tau_d<1.
\end{equation}
Here, we set $\tau_d=0.5$. For PCR, training frequencies pool both orientations of each observed entity pair, and $N$ counts these
oriented views. For RCR, $N$ counts entities appearing in the training graph.

Let $\mathcal C_d$ contain the eligible label comparisons in entity pairs for PCR or entities for RCR. Labels are taken from the union of predictions and training facts, and each comparison involves at least one label assigned by a predicted triple.
For a nonempty comparison pool,
\[
D_d =
100\,
\frac{
\sum_{(u,\{a,b\})\in\mathcal C_d}
\mathbf{1}[\{a,b\}\in\mathcal B_d]
}{
|\mathcal C_d|
}.
\]
Each eligible label combination is counted once in its entity pair or entity.

\paragraph{Gender Role Conflict Rate.}
Under CFamily's conventions, let $\mathcal M$ and
$\mathcal F$ denote male and female role labels,
respectively. Relation direction determines which
entity receives each role. Let
$\mathcal V_g=\{v:S_v\cap(\mathcal M\cup\mathcal F)
\neq\varnothing\}$ be the entities assigned at least
one gendered role. We compute
\begin{equation}
\mathrm{GCR}
=
\frac{100}{|\mathcal V_g|}
\sum_{v\in\mathcal V_g}
\mathbf{1}\!\left[
S_v\cap\mathcal M\neq\varnothing
\;\land\;
S_v\cap\mathcal F\neq\varnothing
\right].
\end{equation}

\paragraph{Parent Cardinality Violation Rate.}
Let $P_v$ be the set of distinct parent entities assigned
to child $v$, after normalizing relation directions.
Let $\mathcal V_p=\{v:|P_v|>0\}$. We compute
\begin{equation}
\mathrm{PCVR}
=
\frac{100}{|\mathcal V_p|}
\sum_{v\in\mathcal V_p}
\mathbf{1}[|P_v|\geq 3].
\end{equation}
Multiple triples identifying the same parent are
counted only once.

\paragraph{Kinship Category Conflict Rate.}
Let $K_p$ contain the kinship categories assigned to
entity pair $p$, and let $\mathcal B_K$ contain the
category pairs deemed incompatible under CFamily's
conventions, including parent--child with sibling or
spousal categories. For
$\mathcal P_k=\{p:K_p\neq\varnothing\}$, we compute
\begin{equation}
\mathrm{KCR}
=
\frac{100}{|\mathcal P_k|}
\sum_{p\in\mathcal P_k}c(K_p,\mathcal B_K).
\end{equation}
Rates with zero denominators are undefined.

\subsection{Analysis of Entity Pair Filtering}
\label{app:pair_filtering}

We examine whether GPHT's entity pair filtering excludes
valid test triples that are predicted by \M{}.
To distinguish filtering losses from differences in candidate
coverage, we restrict the analysis to test triples within
both methods' candidate spaces after subgraph construction
and before GPHT's entity pair filtering.
We use the validation-selected thresholds from the main
evaluation. Predictions are merged across subgraphs using
global entity identifiers, and each triple is counted once.

Table~\ref{tab:pair_filtering} reports three quantities.
\textit{Covered} counts the test triples in this shared
candidate space.
\textit{Filtered} counts those excluded by GPHT's entity pair
filtering, with percentages relative to \textit{Covered}.
A triple is counted as filtered only if its entity pair
is not retained in any GPHT subgraph in which it is a candidate.
\textit{Recovered} counts the filtered triples present in
the final prediction set of \M{}, with percentages relative
to \textit{Filtered}.
Recovery requires an exact match of the head entity,
relation, and tail entity.

The results show that some valid test triples are excluded
before GPHT's relation scoring stage yet are recovered by \M{}.
This provides evidence of a recall bottleneck at the
entity-pair filtering stage under the evaluated settings,
separately from the coverage limitations imposed by
subgraph construction.

\begin{table}[t]
\caption{
Analysis of GPHT's entity pair filtering on the candidate space shared with \M{}. 
}
\label{tab:pair_filtering}
\centering
\small
\setlength{\tabcolsep}{5pt}
\begin{tabular}{llrrr}
\toprule
Dataset & Method  & Covered & Filtered & Recovered \\
\midrule
Wiki79k
& GPHT & 12,090 & 10,081 (83.4\%) & 2,804 (27.8\%) \\
Wiki143k
& GPHT & 20,078 & 12,962 (64.6\%) & 1,151 (8.9\%) \\
CFamily
& GPHT & 4,550 & 3,109 (68.3\%) & 862 (27.7\%) \\
\bottomrule
\end{tabular}
\end{table}

\subsection{Effectiveness Analysis of the Generative Framework}
\label{sec:difftsprepaint}
To validate the effectiveness of the overall framework of DiffTSP, we also adapt the training and sampling strategies from Repaint~\cite{lugmayr2022repaint} to DiffTSP, named as DiffTSP-repaint. 
The new training phase thus does not use the support-query learning paradigm. Instead, we train DiffTSP-repaint to denoise the entire graph. The forward process gradually adds noise to the graph $G_{g_i}$. The loss function for training the denoising network is the same as DiffTSP. 
This trains DiffTSP-repaint for an unconditional graph reconstruction task, without explicitly teaching it to complete a graph from a support set.
During the sampling process, we use the trained denoising network to generate new triples, conditioned on the known graph $G^S$. We create a binary mask $m$ from $G^S$ to distinguish between known and unknown edges:

\begin{equation}
m[i, j, k] =
\begin{cases}
1 & \text{if } (v_i, r_k, v_j) \in G^S \\
0 & \text{otherwise}
\end{cases}
\end{equation}
The known edges are where $m=1$, and the unknown edges are where $m=0$.
For each sampling step from $t$ to $t-1$, we perform three operations:
First, we use the denoising network to predict the state of the unknown edges of the graph:
    \begin{equation}
        G_{t-1}^{\text{unknown}} \sim p_\theta(G_{t-1} | G_t).
    \end{equation}
Next, for the known edges (defined by the support graph $G^S$), we do not use the model's prediction. Instead, we re-introduce them by sampling from the forward process. We take the known edges from $G^S$ and add $t-1$ steps of noise to them, obtaining a noisy graph:
    \begin{equation}
        G_{t-1}^{\text{known}} \sim q(G_{t-1} | G^S).
    \end{equation}
Finally, we combine the two parts using a binary mask $m$ derived from $G^S$. The complete graph for the next step $G_{t-1}$ is constructed as:
    \begin{equation}
        G_{t-1} = (m \odot G_{t-1}^{\text{known}}) + ((1-m) \odot G_{t-1}^{\text{unknown}}),
    \end{equation}
where $\odot$ denotes element-wise multiplication. This combined graph $G_{t-1}$ then serves as the input for the next sampling step.

According to Table~\ref{tab:diffrepaint}, DiffTSP-repaint has a lower performance on $F_{TSP}$. This decline is attributed to its training methodology, which directly restores the entire graph without leveraging the support-query learning paradigm. Without using a support graph as a condition during training, the model's capability to effectively restore incomplete graphs is compromised.
\begin{table}
\caption{Performance of DiffTSP and DiffTSP-repaint in CFamily dataset.}
\label{tab:diffrepaint}
\centering
\scalebox{1}{
\begin{tabular}{lcc>{\columncolor{gray!10}}c}
\toprule
Method& \multicolumn{3}{c}{CFamily}\\
& $JPrecision$&$STRecall$ &$F_{TSP}$  \\
\midrule

DiffTSP-repaint & 0.508 & \textbf{0.704} & 0.590 \\
DiffTSP &\textbf{0.675}&0.600&\textbf{0.635}\\
\bottomrule
\end{tabular}
}
\end{table}



\end{document}